\pdfoutput=1

\documentclass[11pt]{article}
\usepackage[utf8]{inputenc}
\usepackage{fullpage,indentfirst}
\usepackage{mathtools}
\usepackage{amsfonts}
\usepackage{amsthm}
\usepackage{mathrsfs}
\usepackage{xcolor}
\usepackage{bbm}
\usepackage[backref, colorlinks,citecolor=blue,linkcolor=magenta,bookmarks=true]{hyperref}
\usepackage{dsfont}

\newtheorem{theorem}{Theorem}
\newtheorem{definition}{Definition}
\newtheorem{lemma}{Lemma}
\newtheorem{fact}[lemma]{Fact}
\newtheorem{conjecture}{Conjecture}
\newtheorem{remark}[lemma]{Remark}

\newcommand{\1}[1]{\Ind\left[#1\right]} 
\newcommand{\Afree}{A^{\textrm{free}}}
\newcommand{\Aone}{A^{1}}
\newcommand{\Azero}{A^{0}}
\newcommand{\D}{\mathcal{D}} 
\newcommand{\eps}{\varepsilon} 
\newcommand{\Ex}[2]{\operatorname*{\mathbb{E}}_{#1}\left[#2\right]} 
\newcommand{\G}{\mathscr{G}} 
\renewcommand{\P}{\mathcal{P}} 

\newcommand{\puritygain}{\G\text{-}\mathrm{purity}\text{-}\mathrm{gain}}
\newcommand{\pr}[2]{\Pr_{#1}\left[#2\right]} 
\newcommand{\R}{\mathbb{R}} 
\newcommand{\sym}{\operatorname*{\triangle}} 

\newcommand{\zo}{\{0,1\}}
\newcommand{\Ind}{\mathds{1}}

\def\colorful{1}

\ifnum\colorful=1

\fi
\ifnum\colorful=0

\fi

\begin{document}
\title{Decision tree heuristics can fail, even in the smoothed setting
 \vspace{15pt}}

\author{Guy Blanc \vspace{8pt} \\ \hspace{-5pt}{\sl Stanford} \and \hspace{10pt} Jane Lange \vspace{8pt} \\
\hspace{4pt}  {\sl MIT}
\and Mingda Qiao \vspace{8pt}\\ \hspace{-8pt} {\sl Stanford} 
\and Li-Yang Tan \vspace{8pt} \\ \hspace{-8pt} {\sl Stanford}}  

\newcommand{\gnote}[1]{\footnote{{\bf \color{red}Guy}: {#1}}}
\newcommand{\jnote}[1]{\footnote{{\bf \color{blue}Jane}: {#1}}}

\date{\vspace{15pt}\small{\today}}

\maketitle

\begin{abstract} 
    Greedy decision tree learning heuristics are mainstays of machine learning practice, but theoretical justification for their empirical success remains elusive. In fact, it has long been known that there are simple target functions for which they fail badly (Kearns and Mansour, STOC 1996).  

Recent work of Brutzkus, Daniely, and Malach (COLT 2020) considered the smoothed analysis model as a possible avenue towards resolving this disconnect.  Within the smoothed setting and for targets $f$ that are $k$-juntas, they showed that these heuristics successfully learn~$f$ with depth-$k$ decision tree hypotheses.  They conjectured that the same guarantee holds more generally for targets that are depth-$k$ decision trees.

We provide a counterexample to this conjecture: we construct targets that are depth-$k$ decision trees and show that even in the smoothed setting, these heuristics build trees of depth $2^{\Omega(k)}$ before achieving high accuracy.  We also show that the guarantees of Brutzkus et al.~cannot extend to the agnostic setting: there are targets that are very close to $k$-juntas, for which these heuristics build trees of depth $2^{\Omega(k)}$ before achieving high accuracy. 
\end{abstract}

\thispagestyle{empty}

\newpage
\setcounter{page}{1}


\section{Introduction}
Greedy decision tree learning heuristics are among the earliest and most basic algorithms in machine learning.  Well-known examples include ID3~\cite{Qui86}, its successor C4.5~\cite{Qui93}, and CART~\cite{BFSO84}, all of which continue to be widely employed in everyday ML applications.  These simple heuristics build a decision tree for labeled dataset $S$ in a greedy, top-down fashion.  They first identify a ``good" attribute to query as the root of the tree.  This induces a partition of $S$ into $S_0$ and $S_1$, and the left and right subtrees are built recursively using $S_0$ and $S_1$ respectively.  

In more detail, each heuristic is associated with an {\sl impurity function} $\mathscr{G} : [0,1]\to [0,1]$ that is concave, symmetric around $\frac1{2}$, and satisfies $\mathscr{G}(0) = \mathscr{G}(1) = 0$ and $\mathscr{G}(\frac1{2}) = 1$.  Examples include the binary entropy function $\mathscr{G}(p) = \textnormal{H}(p)$ that is used by ID3 and C4.5, and the Gini impurity function $\mathscr{G}(p) = 4p(1-p)$ that is used by CART;  Kearns and Mansour~\cite{KM96} proposed and analyzed the function $\mathscr{G}(p) = 2\sqrt{p(1-p)}$.   For a target function $f : \R^n \to \zo$ and a distribution $\mathcal{D}$ over $\R^n$, these heuristics build a decision tree hypothesis for $f$ as follows: 
\begin{enumerate} 
\item {\sl Split:} Query $\Ind[x_i \ge \theta]$ as the root of the tree, where $x_i$ and $\theta$ are chosen to (approximately) maximize the {\sl purity gain with respect to $\mathscr{G}$}: 
\[ \puritygain_{\D}(f,x_i) \coloneqq \mathscr{G}(\Ex{}{f}) - \big(\pr{}{x_i \ge \theta} \cdot \mathscr{G}(\Ex{}{f_{x_i \ge \theta}}) + \pr{}{x_i < \theta} \cdot \mathscr{G}(\Ex{}{f_{x_i < \theta}})\big),\]
where the expectations and probabilities above are with respect to randomly drawn labeled examples $(x,f(x))$ where $x \sim \mathcal{D}$, and $f_{x_i\ge \theta}$ denotes the restriction of $f$ to inputs satisfying $x_i\ge \theta$ (and similarly for $f_{x_i<\theta}$).
\item {\sl Recurse:}  Build the left and right subtrees by recursing on $f_{x_i \ge \theta}$ and $f_{x_i < \theta}$ respectively.  
\item {\sl Terminate:} The recursion terminates when the depth of the tree reaches a user-specified depth parameter.  Each leaf $\ell$ of the tree is labeled by $\mathrm{round}(\Ex{}{f_\ell})$, where we associate $\ell$ with the restriction corresponding to the root-to-$\ell$ path within the tree and $\mathrm{round}(p) \coloneqq \Ind[p\ge \frac1{2}]$.
\end{enumerate}

Given the popularity and empirical success of these heuristics\footnote{CART and C4.5 were named as two of the ``Top 10 algorithms in data mining" by the International Conference on Data Mining (ICDM) community~\cite{Top10}; other algorithms on this list include $k$-means, $k$-nearest neighbors, Adaboost, and PageRank, all of whose theoretical properties are the subjects of intensive study.  C4.5 has also been described as ``probably the machine learning workhorse most widely used in practice to date"~\cite{WFHP16}.}, it is natural to seek theoretical guarantees on their performance: 
\begin{quote} {\sl Let $f : \mathbb{R}^n \to \{0,1\}$ be a target function and $\mathcal{D}$ be a distribution over $\mathbb{R}^n$. Can we obtain a high-accuracy hypothesis for $f$ by growing a depth-$k'$ tree using these heuristics, where~$k'$ is not too much larger than $k$, the optimal decision tree depth for~$f$?} \hfill{($\diamondsuit$)} 
\end{quote} 

\subsection{Background and prior work}

\paragraph{A simple and well-known impossibility result.} Unfortunately, it has long been known~\cite{KM96,Kea96} that no such guarantee is possible even under favorable feature and distributional assumptions.  Consider the setting of binary features (i.e.~$f : \zo^n\to \zo$) and the uniform distribution $\mathcal{U}$ over $\zo^n$, and suppose $f$ is the parity of two unknown features $x_i \oplus x_j$ for $i,j\in [n]$.   It can be easily verified that for all impurity functions $\mathscr{G}$, all features have the same purity gain:  $\puritygain_{\mathcal{U}}(f,x_\ell) = 0$ for all $\ell\in [n]$, regardless of whether $\ell \in \{i,j\}$.  Therefore,  these heuristics may build a tree of depth $\Omega(n)$, querying irrelevant variables $x_\ell$ where $\ell \notin \{i,j\}$, before achieving any nontrivial accuracy.  This is therefore an example where the target $f$ is computable by a decision tree of depth $k = 2$, and yet these heuristics may build a tree of depth $k' = \Omega(n)$ before achieving any nontrivial accuracy. 

\paragraph{Smoothed analysis.}  In light of such impossibility results, a line of work has focused on establishing provable guarantees for restricted classes of target functions~\cite{FP04,Lee09,BDM19,BLT-ITCS,BLT-ICML}; we give an overview of these results in Section~\ref{sec:related}. 

The focus of our work is instead on {\sl smoothed analysis} as an alternative route towards evading these impossibility results, an approach that was recently considered by Brutzkus, Daniely, and Malach~\cite{BDM20}.   Smoothed analysis is by now a standard paradigm for going beyond worst-case analysis.  Roughly speaking, positive results in this model show that ``hard instances are pathological."  Smoothed analysis has been especially influential in accounting for the empirical effectiveness of algorithms widely used in practice, a notable example being the simplex algorithm for linear programming~\cite{ST04}.  The idea of  analyzing greedy decision tree learning heuristics through the lens of smoothed analysis is therefore very natural.  

A {\sl smoothed product distribution} over $\zo^n$, a notion introduced by Kalai, Samrodnitsky, and Teng~\cite{KST09}, is obtained by randomly and independently perturbing the bias of each marginal of a product distribution. For smoothed product distributions, Brutzkus et al.~proved strong guarantees on the performance of greedy decision tree heuristics when run on targets that are {\sl juntas}, functions that depend only on a small number of its features.  For a given impurity function $\mathscr{G}$, let us write $\mathcal{A}_{\mathscr{G}}$ to denote the corresponding decision tree learning heuristic. 

\begin{theorem}[Performance guarantee for targets that are $k$-juntas~\cite{BDM20}] 
\label{thm:BDM} 
For all impurity functions $\mathscr{G}$ and for all target functions $f: \zo^n \to \zo$ that are $k$-juntas, if $\mathcal{A}_{\mathscr{G}}$ is trained on examples drawn from a smoothed product distribution, it learns a decision tree hypothesis of depth $k$ that achieves perfect accuracy. 
\end{theorem} 

(Therefore Theorem~\ref{thm:BDM} shows that  the smoothed setting enables one to circumvent the impossibility result discussed above, which was based on targets that are $2$-juntas.) 

Every $k$-junta is computable by a depth-$k$ decision tree, but a depth-$k$ decision tree can depend on as many as $2^k$  variables. Brutzkus et al.~left as an open problem of their paper a conjecture that the guarantees of Theorem~\ref{thm:BDM} hold more generally for targets that are depth-$k$ decision trees: 

\begin{conjecture}[Performance guarantee for targets that are depth-$k$ decision trees] 
For all impurity functions $\mathscr{G}$ and for all target functions $f : \zo^n\to\zo$ that are depth-$k$ decision trees, if $\mathcal{A}_{\mathscr{G}}$ is trained on examples drawn from a smoothed product distribution, it learns a decision tree hypothesis of depth $O(k)$ that achieves high accuracy. 
\label{conj:BDM} 
\end{conjecture} 

In other words, Conjecture~\ref{conj:BDM} states that for all targets $f : \zo^n\to\zo$, the sought-for guarantee ($\diamondsuit$) holds if the heuristics are trained on examples drawn from a smoothed product distribution.

\subsection{This work: Lower bounds in the smoothed setting}
 Our main result is a counterexample to Conjecture~\ref{conj:BDM}.  We construct targets that are depth-$k$ decision trees for which all greedy impurity-based heuristics, even in the smoothed setting, may grow a tree of depth $2^{\Omega(k)}$ before achieving high accuracy.  This lower bound is close to being maximally large since Theorem~\ref{thm:BDM} implies an upper bound of $O(2^k)$.  Our result is actually stronger than just a lower bound in the smoothed setting: our lower bound holds with respect to {\sl any} product distribution that is balanced in the sense that its marginals are not too skewed.  

\begin{theorem}[Our main result: a counterexample to Conjecture~\ref{conj:BDM}; informal]  
\label{thm:main-intro} 
Conjecture~\ref{conj:BDM} is false: For all~$k=k(n)$, there are target functions $f : \zo^n\to\zo$ that are depth-$k$ decision trees such that for all impurity functions $\mathscr{G}$, if $\mathcal{A}_{\mathscr{G}}$ is trained on examples drawn from any balanced product distribution, its decision tree hypothesis does not achieve high accuracy unless it has depth $2^{\Omega(k)}$.
\end{theorem}

By building on our proof of Theorem~\ref{thm:main-intro}, we also show that the guarantees of Brutzkus et al.~for $k$-juntas cannot extend to the agnostic setting: 

\begin{theorem}[Theorem~\ref{thm:BDM} does not extend to the agnostic setting; informal] 
\label{thm:agnostic-intro} 
For all $\eps$ and $k=k(n)$, there are target functions $f: \zo^n\to\zo$ that are $\eps$-close to a $k$-junta such that for all impurity functions $\mathscr{G}$, if $\mathcal{A}_{\mathscr{G}}$ is trained on examples drawn from any balanced product distribution, its decision tree hypothesis does not achieve high accuracy unless it has depth $\eps \cdot 2^{\Omega(k)}$.
\end{theorem} 

In particular, there are targets that are $2^{-\Omega(k)}$-close to $k$-juntas, for which these heuristics have to construct a decision tree hypothesis of depth $2^{\Omega(k)}$ before achieving high accuracy. Taken together with the positive result of Brutzkus et al., Theorems~\ref{thm:main-intro} and~\ref{thm:agnostic-intro} add to our understanding of the strength and limitations of  greedy decision tree learning heuristics. 

Our lower bounds are based on new generalizations of the {\sl addressing function}.  Since the addressing function is often a useful extremal example in a variety of settings, we are hopeful that these generalizations and our analysis of them will see further utility beyond the applications of this paper. 
\subsection{Related Work}
\label{sec:related}

As mentioned above, there has been a substantial line of work on establishing provable guarantees for greedy decision tree heuristics when run in restricted classes of target functions.  Fiat and Pechyony~\cite{FP04} considered the class of read-once DNF formulas and halfspaces; the Ph.D.~thesis of Lee~\cite{Lee09} considered the class of monotone functions; Brutzkus, Daniely, and Malach~\cite{BDM19} considered conjunctions and read-once DNF formulas; recent works of~\cite{BLT-ITCS,BLT-ICML} build on the work of Lee and further studied monotone target functions. (All these works focus on the case of binary features and product distributions over examples.)

Kearns and Mansour~\cite{KM96}, in one of the first papers to study these heuristics from a theoretical perspective, showed that they can be viewed as boosting algorithms, with internal nodes of the decision tree hypothesis playing the role of weak learners.  Their subsequent work with Dietterich~\cite{DKM96} provide experimental results that complement the theoretical results of~\cite{KM96}; see also the survey of Kearns~\cite{Kea96}. 

Finally, we mention that decision trees are one of the most intensively studied concept classes in learning theory.  The literature on this problem is rich and vast (see e.g.~\cite{EH89,Riv87,Blu92,Han93,Bsh93,KM93,BFJKMR94,HJLT96,KM99,MR02,JS06,OS07,GKK08,KS06,KST09,KST09,HKY18,CM19,BGLT-NeurIPS1}), studying it from a variety of perspectives and providing both positive and negative results.  However, the algorithms developed in these works do not resemble the greedy heuristics used in practice, and indeed, most of them are  not proper (in the sense of returning a hypothesis that is itself a decision tree).\footnote{Quoting~\cite{KM96}, ``it seems fair to say that despite
their other successes, the models of computational learning theory have not yet provided significant insight
into the apparent empirical success of programs like C4.5 and CART."}

\section{Preliminaries}

Recall that an impurity function $\G: [0, 1] \to [0, 1]$ is concave, symmetric with respect to $\frac{1}{2}$, and satisfies $\G(0) = \G(1) = 0$ and $\G(\frac{1}{2}) = 1$. We further quantify the concavity and smoothness of $\G$ as follows:

\begin{definition}[Impurity functions]\label{def:impurity}
    $\G$ is an $(\alpha, L)$-impurity function if $\G$ is $\alpha$-strongly concave and $L$-smooth, i.e., $\G$ is twice-differentiable and $\G''(x) \in [-L, -\alpha]$ for every $x \in [0, 1]$.
\end{definition}

For a boolean function $f: \{0, 1\}^n \to \{0, 1\}$ and index $i \in [n]$, we write $f_{x_i = 0}$ and $f_{x_i = 1}$ to denote the restricted functions obtained by fixing the $i$-th input bit of $f$ to either $0$ or $1$. Formally, each $f_{x_i = b}$ is a function over $\{0, 1\}^n$ defined as $f_{x_i = b}(x) = f(x^{i\to b})$, where $x^{i \to b}$ denotes the string obtained by setting the $i$-th bit of $x$ to $b$. More generally, a \emph{restriction} $\pi$ is a list of constraints of form ``$x_i = b$'' in which every index $i$ appears at most once. For restriction $\pi = (x_{i_1} = b_1, x_{i_2} = b_2, \ldots)$, the restricted function $f_{\pi}:\{0,1\}^n \to \{0,1\}$ is similarly defined as $f_{\pi}(x) = f(x^{i_1 \to b_1, i_2 \to b_2, \ldots})$.

\begin{definition}[Purity gain]\label{def:purity-gain}
    Let $\D$ be a distribution over $\{0, 1\}^n$ and $p_i = \pr{x \sim \D}{x_i = 1}$. The $\G$-purity gain of querying variable $x_i$ on boolean function $f$ is defined as
    \[
        \puritygain_{\D}(f, x_i)
    \coloneqq
        \G\left(\Ex{x \sim \D}{f(x)}\right)
    -   p_i \G\left(\Ex{x \sim \D}{f_{x_i=1}(x)}\right)
    -   (1 - p_i) \G\left(\Ex{x \sim \D}{f_{x_i=0}(x)}\right).
    \]
\end{definition}

In a decision tree, each node $v$ naturally corresponds to a restriction $\pi_v$ formed by the variables queried by the ancestors of $v$ (excluding $v$ itself). We use $f_{v}$ as a shorthand for $f_{\pi_v}$. We say that a decision tree learning algorithm is \emph{impurity-based} if, in the tree returned by the algorithm, every internal node $v$ queries a variable that maximizes the purity gain with respect to $f_v$.

\begin{definition}[Impurity-based algorithms]\label{def:impurity-algo}
    A decision tree learning algorithm is $\G$-impurity-based if the following holds for every $f: \{0, 1\}^n \to \{0, 1\}$ and distribution $\D$ over $\{0, 1\}^n$: When learning $f$ on $\D$, the algorithm outputs a decision tree such that for every internal node $v$, the variable $x_i$ that is queried at $v$ satisfies $\puritygain_{\D}(f_v, x_i) \ge \puritygain_{\D}(f_v, x_j)$ for every $j \in [n]$.
\end{definition}

The above definition assumes that the algorithm exactly maximizes the $\G$-purity gain at every split, while in reality, the purity gains can only be estimated from a finite dataset. We therefore consider an idealized setting that grants the learning algorithm with infinitely many training examples, which, intuitively, strengthens our lower bounds. (Our lower bounds show that in order for an algorithm to recover a good tree---a high-accuracy hypothesis whose depth is close to that of the target---it would need to query a variable that has {\sl exponentially smaller} purity gain than that of the variable with the largest purity gain. Hence, if purity gains are estimated using finitely many random samples as is done in reality, the strength of our lower bounds imply that with extremely high probability, impurity-based heuristics will fail to build a good tree; see Remark~\ref{remark:finite-sample} for a detailed discussion.)

When a decision tree queries variable $x_i$ on function $f$, it naturally induces two restricted functions $f_{x_i = 0}$ and $f_{x_i = 1}$. The following lemma states that the purity gain of querying $x_i$ is roughly the squared difference between the averages of the two functions, up to a factor that depends on the impurity function $\G$ and the data distribution $\D$. We say that a product distribution over $\{0, 1\}^n$ is \emph{$\delta$-balanced} if the expectation of each of the $n$ coordinates is in $[\delta, 1 - \delta]$.

\begin{lemma}\label{lem:gain-vs-diff}
    For any $f:\{0, 1\}^n \to \{0, 1\}$, $\delta$-balanced product distribution $\D$ over $\{0, 1\}^n$ and $(\alpha, L)$-impurity function $\G$, it holds for $\kappa = \max\left(\frac{2}{\alpha\delta(1-\delta)}, \frac{L}{8}\right)$ and every $i \in [n]$ that
    \[
            \frac{1}{\kappa}
    \le     \frac{\puritygain_{\D}(f, x_i)}{\left[\Ex{x \sim D}{f_{x_i=0}(x)} - \Ex{x \sim D}{f_{x_i=1}(x)}\right]^2}
    \le     \kappa.
    \]
\end{lemma}

\begin{proof}[Proof of Lemma~\ref{lem:gain-vs-diff}]
    Let $p_i = \pr{x \sim \D}{x_i = 1}$ and $\mu_b = \Ex{x \sim D}{f_{x_i = b}(x)}$ respectively. Then, we have $\Ex{x \sim \D}{f(x)} = p_i \mu_1 + (1 - p_i)\mu_0$, and the purity gain can be written as
    \[
        \puritygain_{\D}(f, x_i)
    =   \G(p_i \mu_1 + (1 - p_i)\mu_0) - p_i\G(\mu_1) - (1 - p_i)\G(\mu_0).
    \]
    Since $\G$ is $\alpha$-strongly concave and $L$-smooth, the above is bounded between 
        $\frac{\alpha}{2}\cdot p_i(1 - p_i)\cdot (\mu_0 - \mu_1)^2$
    and
        $\frac{L}{2}\cdot p_i(1 - p_i)\cdot (\mu_0 - \mu_1)^2$.
    Since $\D$ is $\delta$-balanced, we have $\delta(1 - \delta) \le p_i (1 - p_i) \le \frac{1}{4}$. It follows that
    \[
        \frac{\alpha}{2}\cdot\delta(1 - \delta)
    \le \frac{\alpha}{2}\cdot p_i(1 - p_i)
    \le \frac{\puritygain_{\D}(f, x_i)}{(\mu_0 - \mu_1)^2}
    \le \frac{L}{2}\cdot p_i(1 - p_i)
    \le \frac{L}{8}.
    \]
    Thus, the ratio is bounded between $1/\kappa$ and $\kappa$.
\end{proof}

Our lower bounds hold with respect to all $\delta$-balanced product distributions. We compare this to the definition of a $c$-\emph{smoothened} $\delta$-balanced product distribution from \cite{BDM20}.
\begin{definition}[Smooth distributions]
    A $c$-smoothened $\delta$-balanced product distribution is a \emph{random} product distribution over $\zo^n$ where the marginal for the $i^{\text{th}}$ bit is $1$ with probability $\widehat{p_i} + \Delta_i$ for fixed $\widehat{p_i} \in (\delta + c, 1 - \delta - c)$ and $\Delta_i$ drawn i.i.d.\ from  $\mathrm{Uniform}([-c, c])$.
\end{definition}

Since our lower bounds hold against all $\delta$-balanced product distributions, it also holds against all $c$-\emph{smoothened} $\delta$-balanced product distributions.

\section{Proof overview and formal statements of our results} Our goal is to construct a target function that can be computed by a depth-$k$ decision tree, but on which impurity-based algorithms must build to depth $2^{\Omega(k)}$ or have large error. To do so, we construct a decision tree target $T$ where the variables with {\sl largest} purity gain are at the {\sl bottom} layer of $T$ (adjacent to its leaves).  Intuitively, impurity-based algorithms will build their decision tree hypothesis for $T$ by querying all the variables in the bottom layer of $T$ before querying any of the variables higher up in $T$.  Our construction will be such that until the higher up variables are queried, it is impossible to approximate the target with any nontrivial error.  Summarizing informally, we show that impurity-based algorithms build its decision tree hypothesis for our target by querying variables in exactly the ``wrong order".

The starting point of our construction is the well known {\sl addressing function}. For $k\in\mathbb{N}$, the addressing function $f:\zo^{k + 2^k} \to \zo$ is defined as follows: Given ``addressing bits" $z \in \zo^k$ and ``memory bits" $y \in \zo^{2^k}$, the output $f(y,z)$ is the $z^{\text{th}}$ bit of $y$, where ``$z^{\text{th}}$ bit" is computed by interpreting $z$ as a base-2 integer.  Note that the addressing function is computable by a decision tree of depth $k+1$ that first queries the $k$ addressing bits, followed by the appropriate memory bit. 

For our lower bound, we would like the variables with the highest purity gain to  be the memory bits. However, for smoothed product distributions, the addressing bits might have higher purity gain than the memory bits, and impurity-based algorithms might succeed in learning the addressing function.  We therefore modify the addressing function by making each addressing bit the parity of multiple new bits. We show that by making each addressing bit the parity of sufficiently many new bits, we can drive the purity gain of these new bits down to the point where the memory bits have the highest purity gain as desired---in fact, larger than the addressing bits by a multiplicative factor of $e^{\Omega(k)}$. (Making each addressing bit the parity of multiple new bits increases the depth of the target, so this introduces technical challenges we have to overcome in order to achieve the strongest parameters.)

Our main theorem is formally restated as follows.

\begin{theorem}[Formal version of Theorem~\ref{thm:main-intro}]\label{thm:main}
    Fix $L \ge \alpha > 0$ and $\delta \in (0, \frac{1}{2}]$. There are boolean functions $f_1, f_2, \ldots$ such that: (1) $f_k$ is computable by a decision tree of depth $O(k/\delta)$; (2) For every $\delta$-balanced product distribution $\D$ over the domain of $f_k$ and every $(\alpha, L)$-impurity function $\G$, any $\G$-impurity based decision tree heuristic, when learning $f_k$ on  $\D$, returns a tree that has either depth $\ge 2^k$ or an $\Omega(\delta)$ error.
\end{theorem}

An extension of our construction and its analysis shows that the guarantees of Brutzkus et al.~for targets that are $k$-juntas cannot extend to the agnostic setting.  Roughly speaking, while our variant of the addressing function from Theorem~\ref{thm:main} is far from all $k$-juntas, it can be made close to one by fixing most of the memory bits.  We obtain our result by showing that our analysis continues to hold under such a restriction. 

\begin{theorem}[Formal version of Theorem~\ref{thm:agnostic-intro}]\label{thm:agnostic}
Fix $L \ge \alpha > 0$, $\delta \in (0, \frac{1}{2}]$ and $\eps \in (0, 1]$. There are boolean functions $f_1, f_2, \ldots$ such that for every $\delta$-balanced product distribution $\D$ over the domain of $f_k$: (1) $f_k$ is $\eps$-close to an $O(k/\delta)$-junta with respect to $\D$; (2) For every $(\alpha, L)$-impurity function $\G$, any $\G$-impurity based decision tree heuristic, when learning $f_k$ on  $\D$, returns a tree that has either a depth of $\Omega(\eps \cdot 2^k)$ or an $\Omega(1)$ error.
\end{theorem}

\section{Warm-Up: A Weaker Lower Bound}\label{sec:warmup}

We start by giving a simplified construction that proves a weaker version of Theorem~\ref{thm:main}, in which the $O(k/\delta)$ depth in condition~(1) is relaxed to $O(k^2/\delta)$. For integers $c, k \ge 1$, we define a boolean function $f_{c,k}: \{0, 1\}^{ck^2+2^k} \to \{0, 1\}$ as follows. The input of $f_{c,k}$ is viewed as two parts: $ck^2$ \emph{addressing bits} $x_{i,j}$ indexed by $i \in [k]$ and $j \in [ck]$, and $2^k$ \emph{memory bits} $y_a$ indexed by $a \in \{0, 1\}^k$. The function value $f_{c,k}(x, y)$ is defined by first computing $z_i(x) = \bigoplus_{j=1}^{ck}x_{i,j}$ for every $i \in [k]$, and then assigning $f_{c,k}(x, y) = y_{z(x)}$.

In other words, $f_{c,k}$ is a disjoint composition of the $k$-bit addressing function and the parity function over $ck$ bits. Given addressing bits $x$ and memory bits $y$, the function first computes a $k$-bit address by taking the XOR of the addressing bits in each group of size $ck$, and then retrieves the memory bit with the corresponding address. Clearly, $f_{c,k}$ can be computed by a decision tree of depth $ck^2 + 1$ that first queries all the $ck^2$ addressing bits and then queries the relevant memory bit in the last layer.

\subsection{Address is Almost Uniform}
Drawing input $(x, y)$ from a distribution $\D$ naturally defines a distribution over $\{0, 1\}^k$ of the $k$-bit address $z(x) = (z_1(x), z_2(x), \ldots, z_k(x))$. The following lemma states that when $\D$ is a $\delta$-balanced product distribution, the distribution of $z(x)$ is almost uniform in the $\ell_{\infty}$ sense. Furthermore, this almost uniformity holds even if one of the addressing bits $x_{i,j}$ is fixed.

\begin{lemma}\label{lem:almost-uniform}
    Suppose that $c \ge \frac{\ln 5}{\delta}$ and $\D$ is a $\delta$-balanced product distribution over the domain of $f_{c,k}$. Then,
    \[\left|\pr{(x, y) \sim \D}{z(x) = a} - 2^{-k}\right| \le 5^{-k}, \forall a \in \{0, 1\}^k.\]
    Furthermore, for every $i \in [k]$, $j \in [ck]$ and $b \in \{0, 1\}$,
    \[\left|\pr{(x, y) \sim \D}{z(x) = a|x_{i,j} = b} - 2^{-k}\right| \le 5^{-k}, \forall a \in \{0, 1\}^k.\]
\end{lemma}

The proof of Lemma~\ref{lem:almost-uniform} uses the following simple fact, which states that the XOR of independent biased random bits is exponentially close to an unbiased coin flip.

\begin{fact}\label{fact:XOR-bias}
    Suppose that $x_1, x_2, \ldots, x_n$ are independent Bernoulli random variables, each with an expectation between $\delta$ and $1 - \delta$. Then, $\left|\pr{}{x_1 \oplus x_2 \oplus \cdots \oplus x_n = 1} - \frac{1}{2}\right| \le \frac{1}{2}(1 - 2\delta)^n \le \frac{1}{2}\exp(-2\delta n)$.
\end{fact}

\begin{proof}[Proof of Lemma~\ref{lem:almost-uniform}]
    Since $z_i(x) = \bigoplus_{j=1}^{ck}x_{i,j}$ and $\D$ is $\delta$-balanced, Fact~\ref{fact:XOR-bias} gives
    \[
        \left|\pr{(x, y) \sim \D}{z_i(x) = 1} - \frac{1}{2}\right|
    \le \frac{1}{2}\exp(-2\delta ck)
    \le \frac{1}{2}\cdot 5^{-k}.
    \]
    Note that the bits of $z(x)$ are independent, so $\pr{(x, y) \sim \D}{z(x) = a}$ is given by
    \[
        \prod_{i=1}^{k}\pr{(x, y) \sim \D}{z_i(x) = a_i}
    \le \left(\frac{1}{2} + \frac{1}{2}\cdot 5^{-k}\right)^k
    =   2^{-k} \cdot (1 + 5^{-k})^k
    \le 2^{-k} \cdot (1 + (2/5)^k)
    =   2^{-k} + 5^{-k},
    \]
    where the third step applies $(1 + x)^k \le 1 + 2^kx$ for $x \in [0, 1]$ and integers $k \ge 1$.
    Similarly,
    \[
        \pr{(x, y) \sim \D}{z(x) = a}
    \ge \left(\frac{1}{2} - \frac{1}{2}\cdot 5^{-k}\right)^k
    \ge 2^{-k}\cdot (1 - k \cdot 5^{-k})
    \ge 2^{-k} - 5^{-k},
    \]
    where the last two steps apply $(1 - x)^k \ge 1 - kx$ and $k\cdot2^{-k} \le 1$. This proves the first part.

    The proof of the ``furthermore'' part is essentially the same, except that conditioning on $x_{i,j} = b$, $z_i(x)$ becomes the XOR of $ck - 1$ independent bits and $b$. By Fact~\ref{fact:XOR-bias}, we have
    \[
        \left|\pr{(x, y) \sim \D}{z_i(x) = 1|x_{i,j} = b} - \frac{1}{2}\right|
    \le \frac{1}{2}\exp(-2\delta(ck - 1))
    \le \frac{1}{2}\exp(-\delta ck)
    \le \frac{1}{2}\cdot 5^{-k},
    \]
    and the rest of the proof is the same.
\end{proof}

\subsection{Memory Bits are Queried First}

The following technical lemma states that the purity gain of $f_{c,k}$ is maximized by a memory bit, regardless of the impurity function and the data distribution. Therefore, when an impurity-based algorithm (in the sense of Definition~\ref{def:impurity-algo}) learns $f_{c,k}$, the root of the decision tree will always query a memory bit. Furthermore, this property also holds for restrictions of $f_{c,k}$ as long as the restriction only involves the memory bits.

\begin{lemma}\label{lem:memory-first}
    Fix $L \ge \alpha > 0$ and $\delta \in (0, \frac{1}{2}]$. Let $c_0 = \frac{\ln 5}{\delta}$ and $k_0 = \frac{\ln(2\kappa)}{\ln(5/4)} + 1$, where $\kappa$ is chosen as in Lemma~\ref{lem:gain-vs-diff}. The following holds for every function $f_{c,k}$ with $c \ge c_0$ and $k \ge k_0$: For any $(\alpha, L)$-impurity function $\G$, $\delta$-balanced product distribution $\D$ and restriction $\pi$ of size $< 2^k$ that only contains the memory bits of $f_{c,k}$, the purity gain $\puritygain_{\D}((f_{c,k})_{\pi}, \cdot)$ is maximized by a memory bit.
\end{lemma}

\begin{proof}[Proof of Lemma~\ref{lem:memory-first}]
    Fix $c \ge c_0$ and $k \ge k_0$ and shorthand $f$ for $f_{c,k}$. We will prove a stronger claim: with respect to $f_{\pi}$, every memory bit (that is not in $\pi$) gives a much higher purity gain than every addressing bit does.
    
    \paragraph{Purity gain of the memory bits.} Fix a memory bit $y_a$ ($a \in \{0, 1\}^k$) that does not appear in restriction $\pi$. Let $\mu_b = \Ex{(x, y) \sim \D}{f_{\pi, y_a = b}(x, y)}$ for $b \in \{0, 1\}$. By the law of total expectation,
    \begin{align*}
        \mu_b &= \pr{(x, y) \sim \D}{z(x) = a} \cdot \Ex{(x, y) \sim \D}{f_{\pi, y_a = b}(x, y)|z(x) = a}\\
        &\quad\quad+ \pr{(x, y) \sim \D}{z(x) \ne a} \cdot \Ex{(x, y) \sim \D}{f_{\pi, y_a = b}(x, y)|z(x) \ne a}\\
        &=  \pr{(x, y) \sim \D}{z(x) = a} \cdot b + \pr{(x, y) \sim \D}{z(x) \ne a} \cdot \Ex{(x, y) \sim \D}{f_{\pi}(x, y)|z(x) \ne a}.
    \end{align*}
    Here the second step holds since $f_{\pi, y_a = b}(x, y)$ evaluates to $b$ when the address $z(x)$ equals $a$, and $f_{\pi, y_a = b}$ agrees with $f_{\pi}$ when $z(x) \ne a$.
    Since only the first term above depends on $b$, we have
    \[
        |\mu_0 - \mu_1|
    =   \pr{(x, y) \sim \D}{z(x) = a}
    \ge 2^{-k} - 5^{-k}
    \ge \frac{1}{2}\cdot 2^{-k},
    \]
    where the second step follows from $c \ge c_0$ and Lemma~\ref{lem:almost-uniform}.
    Finally, by Lemma~\ref{lem:gain-vs-diff},
        $\puritygain_{\D}(f_{\pi}, y_a) \ge \frac{1}{\kappa}(\mu_0 - \mu_1)^2 \ge \frac{1}{4\kappa}\cdot 2^{-2k}$.

    \paragraph{Purity gain of the addressing bits.} Similarly, we fix an addressing bit $x_{i,j}$ and define the average $\mu_b = \Ex{(x, y) \sim \D}{f_{\pi, x_{i,j} = b}(x, y)}$. Since $\D$ is a product distribution, $\mu_b$ is equal to the conditional expectation $\Ex{(x, y) \sim \D}{f_{\pi}(x, y) | x_{i,j} = b}$. Then, by the law of total expectation, we can write $\mu_b$ as
    \begin{align*}
        \mu_b
    &=  \sum_{a \in \{0, 1\}^k}\pr{(x, y) \sim \D}{z(x) = a|x_{i,j} = b} \cdot \Ex{(x, y) \sim \D}{f_{\pi}(x, y) | z(x) = a, x_{i,j} = b}\\
    &=  \sum_{a \in \{0, 1\}^k}\pr{(x, y) \sim \D}{z(x) = a|x_{i,j} = b} \cdot \Ex{(x, y) \sim \D}{f_{\pi}(x, y) | z(x) = a}.
    \end{align*}
    Here the second step holds since $f_{\pi}(x, y)$ and $x_{i,j}$ are independent conditioning on the address $z(x)$; in other words, once we know the value of $z(x)$, it doesn't matter how $x$ is set in determining the output of $f$.
    
    Let $c_{a}$ denote $\Ex{(x, y) \sim \D}{f_{\pi}(x, y) | z(x) = a}$, and let $\P_b$ be the distribution of $z(x)$ conditioning on $x_{i,j} = b$. Then, $\mu_b$ is exactly given by $\Ex{a \sim \P_b}{c_a}$. Since each $c_a$ is in $[0, 1]$, $|\mu_0 - \mu_1|$ is upper bounded by the total variation distance between $\P_0$ and $\P_1$:
    \begin{align*}
        |\mu_0 - \mu_1|
    &\le \frac{1}{2}\sum_{a \in \{0, 1\}^k}|\P_0(a) - \P_1(a)|\\
    &\le \frac{1}{2}\sum_{a \in \{0, 1\}^k}\left(|\P_0(a) - 2^{-k}| + |\P_1(a) - 2^{-k}|\right)\\
    &\le \frac{1}{2}\cdot 2^{k}\cdot 2\cdot 5^{-k}
    =   (2/5)^k. \tag{Lemma~\ref{lem:almost-uniform}}
    \end{align*}
    Finally, applying Lemma~\ref{lem:gain-vs-diff} shows that
        $\puritygain_{\D}(f_{\pi}, x_{i,j}) \le \kappa (\mu_0 - \mu_1)^2 \le \kappa \cdot (2/5)^{2k}$.

    Recall that $k \ge k_0 > \frac{\ln(2\kappa)}{\ln(5/4)}$, so we have $\kappa \cdot (2/5)^{2k}
    <   \frac{1}{4\kappa}\cdot 2^{-2k}$.
    Therefore, the purity gain of every memory bit outside the restriction is strictly larger than that of any addressing bit, and the lemma follows immediately.
\end{proof}

\begin{remark}\label{remark:finite-sample}
    The proof above bounds the purity gain of each memory bit and each addressing bit by $\Omega((1/2)^{2k})$ and $O((2/5)^{2k})$ respectively. For Lemma~\ref{lem:memory-first} to hold when the purity gains are estimated from a finite dataset, it suffices to argue that each estimate is accurate up to an $O((2/5)^{2k})$ additive error. By a standard concentration argument, to estimate the purity gains for all restriction $\pi$ of size $\le h$, $2^{O(h + k)}$ training examples are sufficient. When applied later in the proof of Theorem~\ref{thm:main}, this finite-sample version of Lemma~\ref{lem:memory-first} would imply that impurity-based algorithms need to build a tree of depth $h$ as soon as the sample size reaches $2^{\Omega(h+k)}$.
\end{remark}

\subsection{Proof of the Weaker Version}

Now we are ready to prove the weaker version of Theorem~\ref{thm:main}. We will apply Lemma~\ref{lem:memory-first} to argue that the tree returned by an impurity-based algorithm never queries an addressing bit (unless all the $2^k$ memory bits have been queried), and then show that every such decision tree must have an error of $\Omega(\delta)$.

\begin{proof}[Proof of Theorem~\ref{thm:main} (weaker version)]
    Fix integer $c \ge \frac{\ln 5}{\delta}$ and consider the functions $f_{c,1}, f_{c,2}, \ldots$. Since each $f_{c,k}$ is represented by a decision tree of depth $ck^2 + 1 = O(k^2/\delta)$, it remains to show that impurity-based algorithms fail to learn $f_{c,k}$. Fix integer $k \ge k_0$ (where $k_0$ is chosen as in Lemma~\ref{lem:memory-first}) and $\delta$-balanced product distribution $\D$ over the domain of $f_{c,k}$. In the following, we use shorthand $f$ for $f_{c,k}$.
    
    \paragraph{Small trees never query addressing bits.} Let $T$ be the decision tree returned by a $\G$-impurity-based algorithm when learning $f$ on $\D$. If $T$ has depth $> 2^k$, we are done, so we assume that $T$ has depth at most $2^k$. We claim that $T$ never queries the addressing bits of $f$. Suppose otherwise, that an addressing bit is queried at node $v$ in $T$, and no addressing bits are queried by the ancestors of $v$. Then, the restriction $\pi_v$ associated with node $v$ only contains the memory bits of $f$. Since $T$ has depth $\le 2^k$, the size of $\pi_v$ is strictly less than $2^k$. Then, by Lemma~\ref{lem:memory-first}, the $\G$-purity gain with respect to $f_v$ is maximized by a memory bit. This contradicts the assumption that the algorithm is $\G$-impurity-based.
    
    \paragraph{Trivial accuracy if no addressing bits are queried.} We have shown that $T$ only queries the memory bits of $f$. We may further assume that $T$ queries \emph{all} the $2^k$ memory bits before reaching any of its leaves, i.e., $T$ is a full binary tree of depth $2^k$. This assumption is without loss of generality because we can add dummy queries on the memory bits to the leaves of depth $< 2^k$, and label all the resulting leaves with the same bit. This change does not modify the function represented by $T$.
    
    Assuming that $T$ is full, every leaf $\ell$ of $T$ is labeled by $2^k$ bits $(c_a)_{a \in \{0, 1\}^k}$, meaning that each memory bit $y_a$ is fixed to $c_a$ on the root-to-$\ell$ path.
    The expectation of the restricted function $f_{\ell}$ is then given by $\mu_{\ell} \coloneqq \Ex{(x, y) \sim \D}{c_{z(x)}}$.
    Clearly, the error of $T$ is minimized when each leaf $\ell$ is labeled with $\1{\mu_{\ell} \ge \frac{1}{2}}$, and the conditional error when reaching leaf $\ell$ is $\min(\mu_{\ell}, 1 - \mu_{\ell})$.
    
    It remains to show that for a large fraction of leaves $\ell$, $\mu_{\ell}$ is bounded away from $0$ and $1$, so that $\min(\mu_{\ell}, 1 - \mu_{\ell})$ is large. When leaf $\ell$ is randomly chosen according to distribution $\D$, the corresponding $\mu_{\ell}$ is given by
    \begin{equation}\label{eq:leaf-mean}
        \mu_{\ell}
    =   \sum_{a \in \{0, 1\}^k}\pr{(x, y) \sim \D}{z(x) = a}\cdot c_a,
    \end{equation}
    where $(c_a)_{a \in \{0, 1\}^k}$ are $2^k$ independent Bernoulli random variables with means in $[\delta, 1 - \delta]$.
    
    By Lemma~\ref{lem:almost-uniform} and our choice of $c \ge c_0$, $\pr{(x, y) \sim \D}{z(x) = a} \le 2 \cdot 2^{-k}$ holds for every $a \in \{0, 1\}^k$. Thus, each term in~\eqref{eq:leaf-mean} is bounded between $0$ and $2\cdot 2^{-k}$. Furthermore, since each $c_a$ has expectation at least $\delta$, $\Ex{}{\mu_{\ell}} \ge \delta$. Then, Hoeffding's inequality guarantees that over the random choice of $(c_a)_{a \in \{0, 1\}^k}$, $\mu_{\ell} \ge \delta / 2$ holds with probability at least
        $1 - \exp\left(-\frac{2\cdot(\delta/2)^2}{2^k\cdot(2\cdot 2^{-k})^2}\right)
    =   1 - \exp(-2^k\delta^2/8)$,
    which is lower bounded by $2/3$ for all sufficiently large $k$. By a symmetric argument, $\mu_{\ell} \le 1 - \delta / 2$ also holds with probability $\ge 2/3$. Therefore, with probability $\ge 1/3$ over the choice of leaf $\ell$, $\mu_{\ell} \in [\delta / 2, 1 - \delta / 2]$ holds and thus the conditional error on leaf $\ell$ is at least $\delta/2$. This shows that the error of $T$ over distribution $\D$ is lower bounded by $\delta/6$, which completes the proof.
\end{proof}

\section{Proof of Theorem~\ref{thm:main}}\label{sec:main-proof}
When proving the weaker version of Theorem~\ref{thm:main}, each hard instance $f_{c,k}$ has $\Theta(k^2)$ addressing bits grouped into $k$ disjoint subsets, and the $k$-bit address is defined by the XOR of bits in each subset. We will prove Theorem~\ref{thm:main} using a slightly different construction that computes address from $k$ overlapping subsets of only $O(k)$ addressing bits.

For integers $c, k \ge 1$ and a list of $k$ sets $S = (S_1, S_2, \ldots, S_k)$ where each $S_i \subseteq [ck]$, we define a boolean function $f_{c,k,S}: \{0, 1\}^{ck + 2^k} \to \{0, 1\}$ as follows. The input of $f_{c,k,S}$ is again divided into two parts: $ck$ addressing bits $x_1, x_2, \ldots, x_{ck}$ and $2^k$ memory bits $y_a$ indexed by a $k$-bit address $a$. The function value $f(x, y)$ is computed by taking $z_i(x) = \bigoplus_{j \in S_i}x_j$ and then $f(x, y) = y_{z(x)}$. Clearly, $f_{c,k,S}$ can be computed by a decision tree of depth $ck + 1$ that first queries all the $ck$ addressing bits $x_1, x_2, \ldots, x_{ck}$, and then queries the relevant memory bit $y_{z(x)}$.

Let $\sym_{i=1}^{k}S_i$ denote the $k$-ary symmetric difference of sets $S_1$ through $S_k$, i.e., the set of elements that appear in an odd number of sets. We say that a list of sets $S = (S_1, S_2, \ldots, S_k)$ has \emph{distance} $d$, if any non-empty collection of sets has a symmetric difference of size at least $d$, i.e., $\left|\sym_{i \in I}S_i\right| \ge d$ for every non-empty $I \subseteq [k]$. In the following, we prove analogs of Lemmas \ref{lem:almost-uniform}~and~\ref{lem:memory-first} for function $f_{c,k,S}$ assuming that $S$ has a large distance; Theorem~\ref{thm:main} would then follow immediately. 

\begin{lemma}\label{lem:almost-uniform-overlap}
    Suppose that $\D$ is a $\delta$-balanced product distribution over the domain of $f_{c,k,S}$ and $S$ has distance $d \ge \frac{\ln 5}{\delta}\cdot k$. Then,
    \[\left|\pr{(x, y) \sim \D}{z(x) = a} - 2^{-k}\right| \le 5^{-k}, \forall a \in \{0, 1\}^k.\]
    Furthermore, for every $i \in [ck]$ and $b \in \{0, 1\}$,
    \[\left|\pr{(x, y) \sim \D}{z(x) = a|x_i = b} - 2^{-k}\right| \le 5^{-k}, \forall a \in \{0, 1\}^k.\]
\end{lemma}

We prove Lemma~\ref{lem:almost-uniform-overlap} by noting that the distribution of $z(x)$ has exponentially small Fourier coefficients (except the degree-$0$ one) under the assumptions, and is thus close to the uniform distribution over $\{0, 1\}^k$. More concretely, our goal is to show that, for every $I \subseteq [k]$ the quantity $\bigoplus_{i \in I}z_i(x)$ is $1$ with probability nearly exactly $\frac{1}{2}$. Afterwards, we will show this is sufficient to guarantee that the distribution of $z(x)$ is close to the uniform distribution.

\begin{proof}[Proof of Lemma~\ref{lem:almost-uniform-overlap}]
    Since $z_i(x) = \bigoplus_{j \in S_i}x_j$, we have $\bigoplus_{i \in I}z_i(x) = \bigoplus_{j \in S_I}x_j$ for every $I \subseteq [k]$, where $S_I = \sym_{i \in I}S_i$ is the symmetric difference of the corresponding sets. Since $S$ has distance $d$, $|S_I| \ge d$ for every non-empty $I \subseteq [k]$ and thus $\bigoplus_{i \in I}z_i(x)$ is the XOR of at least $d$ independent bits. Note that $1 - 2\bigoplus_{i \in I}z_i(x) = \prod_{i \in I}(1 - 2z_i(x))$. By Fact~\ref{fact:XOR-bias} and $d \ge \frac{\ln 5}{\delta}\cdot k$,
    \begin{equation}\label{eq:product-bound}
        \left|\Ex{(x, y) \sim \D}{\prod_{i \in I}(1 - 2z_i(x))}\right|
    =   2\cdot \left|\pr{(x, y) \sim \D}{\bigoplus_{i \in I}z_i(x) = 1} - \frac{1}{2}\right|
    \le \exp(-2\delta d)
    \le 5^{-k}.
    \end{equation}
    Note that for $b_1, b_2 \in \{0, 1\}$, we have $\1{b_1 = b_2} = \frac{(1 - 2b_1)(1 - 2b_2) + 1}{2}$. Therefore, for every $a \in \{0, 1\}^k$,
    \begin{align*}
        \left|\pr{(x, y) \sim \D}{z(x) = a} - 2^{-k}\right|
    &=  \left|\Ex{(x, y) \sim \D}{\prod_{i=1}^{k}\frac{(1 - 2a_i)(1 - 2z_i(x)) + 1}{2}} - 2^{-k}\right|\\
    &=  2^{-k}\left|\sum_{I \subseteq [k]}\Ex{(x, y) \sim \D}{\prod_{i\in I}(1 - 2a_i)(1 - 2z_i(x))} - 1\right| \tag{expansion of product and linearity}\\
    &=  2^{-k}\left|\sum_{I \subseteq [k]: I\ne\emptyset}\Ex{(x, y) \sim \D}{\prod_{i\in I}(1 - 2a_i)(1 - 2z_i(x))}\right| \tag{empty product equals $1$}\\
    &\le 2^{-k}\sum_{I \subseteq [k]:I\ne\emptyset}\left|\prod_{i\in I}(1 - 2a_i)\right|\cdot \left|\Ex{(x, y) \sim \D}{\prod_{i\in I}(1 - 2z_i(x))}\right| \tag{triangle inequality and linearity}\\
    &=  2^{-k}\sum_{I \subseteq [k]:I\ne\emptyset}\left|\Ex{(x, y) \sim \D}{\prod_{i\in I}(1 - 2z_i(x))}\right| \tag{$|1 - 2a_i| = 1$}\\
    &\le 2^{-k}\cdot (2^k - 1) \cdot 5^{-k} < 5^{-k}. \tag{Inequality~\eqref{eq:product-bound}}
    \end{align*}
    
    The proof of the ``furthermore'' part is the same, except that after conditioning on $x_i = b$, each $\bigoplus_{j \in I}z_j(x)$ is now the XOR of at least $d - 1$ independent bits, and the remaining proof goes through.
\end{proof}

We note that the proof of Lemma~\ref{lem:memory-first} depends on the definition of $z(x)$ only through the application of Lemma~\ref{lem:almost-uniform}. Thus, Lemma~\ref{lem:almost-uniform-overlap} directly implies the following analog of Lemma~\ref{lem:memory-first}:

\begin{lemma}\label{lem:memory-first-overlap}
    Fix $L \ge \alpha > 0$ and $\delta \in (0, \frac{1}{2}]$. Let $c_0 = \frac{\ln 5}{\delta}$ and $k_0 = \frac{\ln(2\kappa)}{\ln(5/4)} + 1$, where $\kappa$ is chosen as in Lemma~\ref{lem:gain-vs-diff}. The following holds for every function $f_{c,k,S}$ such that $k \ge k_0$ and $S$ has distance $c_0 k$: For any $(\alpha, L)$-impurity function $\G$, $\delta$-balanced product distribution $\D$ and restriction $\pi$ of size $< 2^k$ that only contains the memory bits of $f_{c,k,S}$, the purity gain $\puritygain_{\D}((f_{c,k,S})_{\pi}, \cdot)$ is maximized by a memory bit.
\end{lemma}

Finally, we prove Theorem~\ref{thm:main} by showing the existence of a set family $S$ with a good distance.

\begin{proof}[Proof of Theorem~\ref{thm:main}]
    Fix $\delta \in (0, \frac{1}{2}]$. The Gilbert--Varshamov bound for binary linear codes implies that for some $c = \Theta(1/\delta)$, there exists a binary linear code with rate $\frac{1}{c}$ and relative distance $\frac{\ln 5}{\delta c}$. It follows that for every sufficiently large $k$, there exists $S^{(k)} = (S^{(k)}_1, S^{(k)}_2, \ldots, S^{(k)}_k)$ such that each $S^{(k)}_i \subseteq [ck]$ and $S^{(k)}$ has distance $\frac{\ln 5}{\delta} \cdot k$. This can be done by using the $i$-th basis of the linear code as the indicator vector of subset $S^{(k)}_i$ for each $i \in [k]$.
    
    We prove Theorem~\ref{thm:main} using functions $f_{c,1,S^{(1)}}, f_{c,2,S^{(2)}}, \ldots$. Since each $f_{c,k,S^{(k)}}$ can be represented by a decision tree of depth $ck + 1 = O(k/\delta)$, it remains to prove that impurity-based algorithms fail to learn $f_{c,k,S^{(k)}}$. Lemma~\ref{lem:memory-first-overlap} guarantees that the tree returned by such algorithms either has depth $> 2^k$, or never queries any addressing bits. In the latter case, by the same calculation as in the proof of the weaker version, the decision tree must have an $\Omega(\delta)$ error on distribution $\D$.
\end{proof}

\section{Proof of Theorem~\ref{thm:agnostic}}



We prove Theorem~\ref{thm:agnostic} using the construction of $f_{c,k,S}$ in Section~\ref{sec:main-proof}, where $S = (S_1, S_2, \ldots, S_k)$ is a list of $k$ subsets of $[ck]$ and each $S_i$ specifies how the i-th bit of the address, $z_i(x)$, is computed from the addressing bits $x_1$ through $x_{ck}$. Note that $f_{c,k,S}$ itself depends on $\Omega(2^k)$ input bits and is thus not an $O(k)$-junta. Nevertheless, we will show that, after we fix most of the memory bits of $f_{c,k,S}$, the function is indeed close to a $(ck)$-junta with relevant inputs being the $ck$ addressing bits. Then, as in the proof of Theorem~\ref{thm:main}, we will argue that impurity-based heuristics still query the (unfixed) memory bits before querying any of the addressing bits, resulting in a tree that is either exponentially deep or far from the target function.

\begin{proof}[Proof of Theorem~\ref{thm:agnostic}]
    As in the proof of Theorem~\ref{thm:main}, we can find functions $f_{c,1,S^{(1)}}, f_{c,2,S^{(2)}}, \ldots$ for some $c = \Theta(1/\delta)$ such that each $S^{(k)}$ has distance $\ge \frac{\ln 5}{\delta}\cdot k$. We fix a sufficiently large integer $k$ and shorthand $f$ for $f_{c,k,S^{(k)}}$ in the following.
    
    Partition $\{0, 1\}^k$ into three sets $\Azero$, $\Aone$ and $\Afree$ such that $|\Azero| = |\Aone|$ and $\eps\cdot 2^{k-2} \le |\Afree| \le \eps\cdot 2^{k-1}$. Consider the restriction $\pi$ of function $f$ such that the memory bit $y_a$ is fixed to be $0$ for every $a \in \Azero$ and fixed to be $1$ for every $a \in \Aone$; the memory bits with addresses in $\Afree$ are left as ``free'' variables. We will prove the theorem using $f_{\pi}$ as the $k$-th function in the family.
    
    \paragraph{$f_{\pi}$ is close to a junta.} Consider the function $g: \{0,1\}^{ck + 2^k} \to \{0, 1\}$ defined as $g(x, y) = \1{z(x) \in \Aone}$, where $z(x)$ denotes $(z_1(x), z_2(x), \ldots, z_k(x))$ and each $z_i(x) = \bigoplus_{j \in S^{(k)}_i}x_j$. Clearly, $g(x, y)$ only depends on $x \in \{0, 1\}^{ck}$ and is thus a $(ck)$-junta. Furthermore, for every input $(x, y)$ such that $z(x) \in \Azero$ (resp.\ $z(x) \in \Aone$), both $f_{\pi}$ and $g$ evaluate to $0$ (resp.\ $1$). Thus, $f_{\pi}$ and $g$ may disagree only if $z(x) \in \Afree$. It follows that for every $\delta$-balanced product distribution $\D$,
    \begin{align*}
        \pr{(x, y) \sim \D}{f_{\pi}(x, y) \ne g(x, y)}
    &\le \pr{(x, y) \sim \D}{z(x) \in \Afree}\\
    &\le |\Afree| \cdot (2^{-k} + 5^{-k}) \tag{Lemma~\ref{lem:almost-uniform-overlap}}\\
    &\le \eps\cdot 2^{k-1}\cdot (2^{-k} + 5^{-k}) < \eps. \tag{$|\Afree| \le \eps\cdot 2^{k-1}$}
    \end{align*}
    Therefore, $f_{\pi}$ is $\eps$-close to an $O(k/\delta)$-junta (namely, $g$) with respect to distribution $\D$.
    
    \paragraph{Impurity-based algorithms fail to learn $f_{\pi}$.} Let $T$ be the decision tree returned by an $\G$-impurity based algorithm when learning $f_{\pi}$ on distribution $\D$. By Lemma~\ref{lem:memory-first-overlap}, $T$ must query all the free memory bits with addresses in $\Afree$ before querying any of the addressing bits. Thus, either $T$ has depth $> |\Afree| = \Omega(\eps\cdot2^k)$, or $T$ only queries the free memory bits of $f_{\pi}$.
    
    In the latter case, we may again assume without loss of generality that $T$ queries all the free memory bits $(y_a)_{a \in \Afree}$ before reaching any of its leaves, i.e., $T$ is a full binary tree of depth $|\Afree|$. Then, every leaf $\ell$ naturally specifies $2^k$ bits $(c_a)_{a \in \{0, 1\}^k}$ defined as
    \[
        c_a = \begin{cases}
            0, & a \in \Azero,\\
            1, & a \in \Aone,\\
            b, & a \in \Afree, y_a\text{ is fixed to }b\text{ on the root-to-}\ell\text{ path.}
        \end{cases}
    \]
    Let $\mu_{\ell} \coloneqq \Ex{(x, y) \sim \D}{c_{z(x)}}$. Again, the minimum possible error conditioning on reaching leaf $\ell$ is  $\min(\mu_{\ell}, 1 - \mu_{\ell})$, achieved by labeling $\ell$ with $\1{\mu_{\ell} \ge \frac{1}{2}}$. On the other hand, we have
    \begin{align*}
        \mu_{\ell}
    &\ge \pr{(x, y) \sim \D}{z(x) \in \Aone}\\
    &\ge |\Aone|\cdot(2^{-k} - 5^{-k}) \tag{Lemma~\ref{lem:almost-uniform-overlap}}\\
    &\ge  \frac{2^k - |\Afree|}{2}\cdot 2^{-(k+1)} \tag{$2|\Aone| + |\Afree| = 2^k$}\\
    &\ge \frac{2^k - 2^{k-1}}{2}\cdot 2^{-(k+1)} \tag{$|\Afree| \le \eps\cdot 2^{k-1} \le 2^{k-1}$}
    = \frac{1}{8},
    \end{align*}
    and a similar calculation shows $\mu_{\ell} \le \frac{7}{8}$. We conclude that the error of the decision tree $T$ over distribution $\D$ is at least $\frac{1}{8} = \Omega(1)$.
\end{proof}

\section{Conclusion}

We have constructed target functions for which greedy decision tree learning heuristics fail badly, even in the smoothed setting.  Our lower bounds complement and strengthen the parity-of-two-features example discussed in the introduction, which showed that these heuristics fail badly in the non-smoothed setting.  

It can be reasonably argued that real-world data sets do not resemble the target functions considered in this paper or the parity-of-two-features example.  Perhaps the sought-for guarantee $(\diamondsuit)$, while false for certain target functions even in the smoothed setting, is nonetheless true for broad and natural classes of targets?  It would be interesting to reexamine, through the lens of smoothed analysis, provable guarantees for restricted classes of functions that have been established.  For example, can the guarantees of~\cite{BLT-ITCS,BLT-ICML} for monotone target functions and product distributions be further strengthened in the smoothed setting? The target functions considered in this paper, as well as the parity-of-two-features example, are non-monotone.

\section*{Acknowledgements}

We thank the anonymous reviewers, whose suggestions have helped improved this paper. 

\bibliographystyle{alpha}
\bibliography{main}

\end{document}